\documentclass{article}

\usepackage{xcolor}
\usepackage{amsmath} 
\usepackage{graphicx}
\usepackage{subcaption}
\usepackage{graphicx}
\usepackage{wrapfig}

\usepackage{booktabs}
\usepackage{graphicx}
\usepackage[table]{xcolor}
\usepackage{multirow}

\definecolor{groupgray}{RGB}{245,245,245}
\definecolor{oursblue}{RGB}{232,242,255}
\definecolor{bestgreen}{RGB}{0,128,0}

\newcommand{\best}[1]{\textbf{#1}}
\newcommand{\second}[1]{\underline{#1}}

\definecolor{toolcallcolor}{HTML}{7B1FA2}
\definecolor{toolrespcolor}{HTML}{1565C0}
\definecolor{answercolor}{HTML}{2E7D32}
\definecolor{toolblockcolor}{HTML}{EF6C00}

\newcommand{\ctt}[2]{%
  \ifmmode
    \text{\textcolor{#1}{\texttt{#2}}}%
  \else
    \textcolor{#1}{\texttt{#2}}%
  \fi
}

\newcommand{\ToolsTok}{\ctt{toolblockcolor}{<tools>...</tools>}}
\newcommand{\ToolCallTok}{\ctt{toolcallcolor}{<tool\_call>}}
\newcommand{\EndToolCallTok}{\ctt{toolcallcolor}{</tool\_call>}}
\newcommand{\ToolRespTok}{\ctt{toolrespcolor}{<tool\_response>}}
\newcommand{\EndToolRespTok}{\ctt{toolrespcolor}{</tool\_response>}}
\newcommand{\AnswerTok}{\ctt{answercolor}{<answer>}}
\newcommand{\EndAnswerTok}{\ctt{answercolor}{</answer>}}
 \usepackage[preprint]{neurips_2026}

\usepackage[utf8]{inputenc} 
\usepackage[T1]{fontenc}    
\usepackage{hyperref}       
\usepackage{url}            
\usepackage{booktabs}       
\usepackage{amsfonts}       
\usepackage{nicefrac}       
\usepackage{microtype}      
\usepackage{xcolor}         

\usepackage{amsmath}
\usepackage{amssymb}
\usepackage{mathtools}
\usepackage{amsthm}
\usepackage[normalem]{ulem}
\useunder{\uline}{\ul}{}
\usepackage[capitalize,noabbrev]{cleveref}
\theoremstyle{plain}
\newtheorem{theorem}{Theorem}[section]
\newtheorem{proposition}[theorem]{Proposition}

\theoremstyle{definition}

\theoremstyle{remark}

\title{Mind the Tool Failures: Achieving Synergistic Tool Gains for Medical Agents}
\author{%
  Yunhui Gan$^{1,2,3}$\thanks{Equal contribution.},~
  Tan Pan$^{1,2,5}$\footnotemark[1],~
  Kaiyu Guo$^{4,2}$\thanks{Corresponding authors.},~
  Limei Han$^{1,2}$ \\
  \textbf{Weimiao Yu$^{5}$,~
  Guangnan Ye$^{1,3}$,~
  Chen Jiang$^{1,2}$\footnotemark[2],~
  Yuan Cheng$^{1,2,3}$\footnotemark[2]} \\
  \textsuperscript{1} Fudan University\quad
  \textsuperscript{2} Shanghai Academy of Artificial Intelligence for Science \\
  \textsuperscript{3} Shanghai Innovation Institute \quad
  \textsuperscript{4} The University of Queensland \\
    \textsuperscript{5} Bioinformatics Institute (BII),
    Agency for Science, Technology and Research (A*STAR)\\
  {\tt\small \{yhgan23,pant23\}@m.fudan.edu.cn} \\
  {\tt\small kaiyu.guo@uq.edu.au, jiangchen@sais.org.cn, cheng\_yuan@fudan.edu.cn}
}

\begin{document}

\maketitle

\begin{abstract}


Medical AI agents increasingly use external tools for diagnosis, treatment recommendation, and evidence retrieval, yet most existing approaches assume that task-appropriate tools are reliable within their intended scope. This assumption is fragile in real clinical settings, where even relevant tools may fail on challenging instances and lead to unsafe downstream decisions. To address this issue, we study medical tool use under imperfect-tool settings to correct failure instances missed by individual tools. Instance-dependent failure patterns create a gap between the best fixed single tool and an ideal instance-wise selector, which we refer to as the Single-Oracle risk gap. The core challenge is that conventional task-level tool selection cannot realize this gap, as it is inherently bounded by the performance of the best single tool. Motivated by this observation, we therefore account for instance-level heterogeneity and formulate tool use as an instance-level selection problem. Particularly, we propose a GRPO-based reinforcement learning framework with rewards for probabilistic risk minimization and disagreement-aware synergy learning, which promotes instance-level correction of erroneous tool consensus. Furthermore, an entropy-guided sampling strategy is adopted to upweight high-disagreement instances, which provide stronger signals for learning instance-specific tool synergy. These two components complement each other in mitigating instance-level heterogeneity and improving tool synergy. Experiments on two tasks and seven medical benchmarks show that our method consistently achieves robust and stable improvements over a broad range of baselines, highlighting the importance of synergy-aware tool use for reliable medical agentic systems.

\end{abstract}


\section{Introduction}

Recent advances in agentic systems have increasingly emphasized the ability of agents to use, select, and evolve external tools~\cite{yao2022react,schick2023toolformer,fallahpour2025medrax}. Existing studies~\cite{shen2023hugginggpt,qian2025toolrl} typically focus on identifying the appropriate tool for a given task or decomposing complex tasks into calls to tools with distinct functional roles. This paradigm has enabled agents to extend their capabilities and solve tasks that require external computation, retrieval, or specialized operations~\cite{paranjape2023art,wang2025opencua}.


While effective in many settings, existing tool-use paradigms often assume that tools are reliable within their intended scope, overlooking failures caused by instance-level heterogeneity. This issue is especially critical in medical agents, where tool failures may lead to unsafe downstream decisions~\cite{habli2020artificial,fallahpour2025medrax}. For example, in chest X-ray diagnosis, a tool that detects cardiomegaly may fail on opacity or pleural effusion, while another tool may exhibit the opposite pattern. Such non-overlapping failures stem from heterogeneous clinical instances and tool-specific biases induced by different architectures and training data. As a result, single-tool reliance is risky, but these heterogeneous failure patterns also reveal potential complementarity across tools. As shown in Figure~\ref{fig:fig1}(a), each tool uniquely solves instances missed by others, indicating that no single tool is uniformly dominant and that exploiting instance-level complementarity can recover additional correct predictions beyond the best individual tool.  We visualize the tool's marginal Oracle\footnote{An Oracle is an idealized evaluator that knows the correct answer and is used to estimate the best possible performance among the available candidates~\cite{malmasi2015oracle}.} accuracy on datasets CheXpert~\cite{irvin2019chexpert} and MIMIC-CXR~\cite{johnson2019mimic}. As provided in Figure~\ref{fig:fig1} (b), the Oracle accuracy increases steadily as more complementary tools are added, substantially exceeding the best single-tool performance. This further indicates that the tool set contains instance-level complementarity and that adding tools with distinct failure patterns can progressively improve the Oracle upper bound.


 Instance-dependent failure patterns create a gap between the best fixed single tool and an ideal instance-wise selector. Prior task-level tool selection methods~\cite{yao2022react,qian2025toolrl, fallahpour2025medrax} focus on aggregate tool performance, overlooking instance-level heterogeneity within the task distribution. As shown in Figure~\ref{fig:fig1}, consequently, it fails to fully realize the potential gains indicated by the Oracle performance, motivating an instance-level policy that exploits tool complementarity for each individual instance.

Based on these observations, we analyze tool behavior from both aggregate- and instance-level perspectives. We define the Single-Oracle risk gap to characterize the performance difference between the best single-tool policy and the Oracle tool set. Motivated by this risk gap, we propose CSRL, a Collaborative Synergy Reinforcement Learning framework built upon GRPO~\cite{shao2024deepseekmath}. CSRL optimizes rewards designed for probabilistic risk minimization and disagreement-aware synergy learning. Specifically, we introduce a Brier reward for probabilistic risk minimization under the Brier loss, together with an override reward that encourages correct decisions in disagreement instances and promotes synergy learning. Furthermore, we introduce Entropy-Guided Sampling, which constructs heterogeneous training batches by upweighting high-disagreement instances, thereby providing stronger supervision for instance-specific synergy. Experimental results of CSRL show that constructed tool sets can achieve synergistic gains beyond the strongest standalone tool. Tools with lower average performance may still provide valuable marginal contributions by covering instances missed by stronger tools. Meanwhile, performance improvements exhibit diminishing returns as the tool set grows, indicating a tradeoff between gains and inference cost. 

\begin{figure}[t]
\centering
\includegraphics[width=\textwidth]{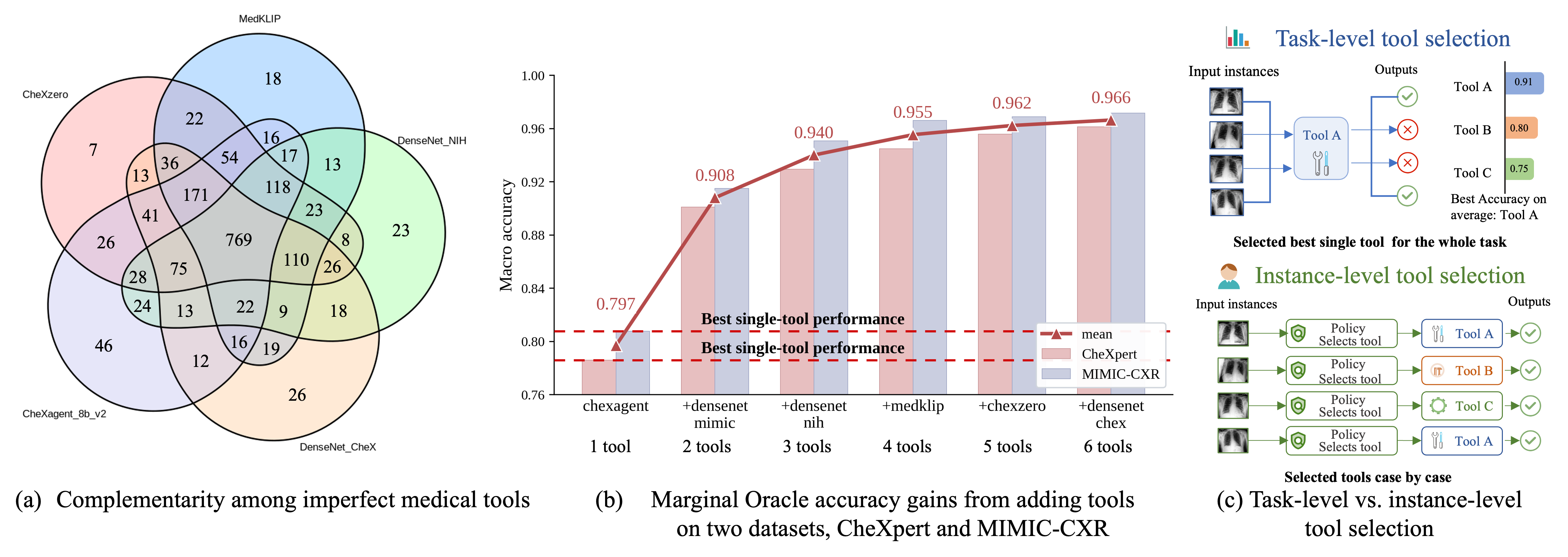}
\caption{(a) Analysis of tool complementarity, where each number indicates the number of instances correctly identified by the corresponding tools;(b) marginal Oracle gains on CheXpert and MIMIC-CXR;(c) Difference between task-level and instance-level tool selection.
}
\label{fig:fig1}
\vspace{-2ex}
\end{figure}

Our contributions can be summarized as follows:
\begin{itemize}
    \item We investigate tool failures in medical agent systems and study how to construct a synergistic set of imperfect tools. Scalability analyses across tools, models, and datasets provide practical guidelines for tool, dataset, and model selection.
    \item Guided by the Single-Oracle risk, we develop CSRL, a collaborative synergy reinforcement learning framework that combines entropy-guided sampling with rewards for probabilistic risk minimization and disagreement-aware synergy learning, enabling performance beyond the best single tool.
    \item Across two tasks and seven datasets, learned synergy consistently outperforms individual tools, static baselines, and both general-purpose and specialized multimodal large language models, with average gains of 7.5\% in accuracy and 6.8\% in F1.
\end{itemize}



\section{Related Work}
\subsection{Tool-Augmented Agents and Tool Selection}
Recent progress in agentic systems has enabled multimodal large language models (MLLMs) to interact with external tools~\cite{schick2023toolformer,yao2022react}, APIs~\cite{patil2024gorilla,li2023apibank}, and specialized models~\cite{shen2023hugginggpt,wu2023visualgpt} to solve tasks beyond their parametric knowledge. Early work, such as ReAct~\cite{yao2022react} and Toolformer~\cite{schick2023toolformer}, showed that language models can interleave reasoning with tool invocation, while subsequent systems, including HuggingGPT~\cite{shen2023hugginggpt}, ART~\cite{paranjape2023art}, Search-R1~\cite{jin2025search}, and ToolRL~\cite{qian2025toolrl}, further explored tool planning, multi-step execution, and reinforcement learning for tool use. These methods typically focus on deciding which tool to call for a given task or how to decompose a complex instruction into a sequence of tool invocations.\par
\subsection{Medical Vision-Language Models and Diagnostic Tools}
Medical AI agents have increasingly incorporated external diagnostic models, retrieval systems, and domain-specific reasoning modules to support clinical question answering, radiology interpretation, triage, and treatment recommendation \citep{fallahpour2025medrax}. In chest X-ray analysis, specialized models such as CheXzero~\cite{tiu2022expert}, MedKLIP~\cite{wu2023medklip}, CheXagent~\cite{chen2024chexagent}, and supervised DenseNet variants~\cite{cohen2020torchxrayvision} provide useful diagnostic signals for different diseases and datasets \cite{ rajpurkar2017chexnet}. More recently, general-purpose and medical Multimodal Large Language Models (MLLMs) have shown promising performance on multimodal medical reasoning tasks \citep{chen2024huatuogptvisioninjectingmedicalvisual, sellergren2025medgemma, xu2025lingshu}.

\subsection{Reinforcement Learning for Tool-use Agents}

Reinforcement learning has recently been adopted to improve the decision-making behavior of language-model agents in settings involving discrete actions, structured outputs, delayed feedback, and non-differentiable external environments. In search-augmented reasoning, policy optimization is used to train agents to issue search queries, incorporate retrieved evidence, and generate final answers through multi-step interactions \citep{jin2025search}. In tool-use settings, reward-driven optimization further enables language agents to learn when and how to invoke external tools based on task feedback, rather than relying solely on supervised demonstrations or hand-crafted tool-use traces \citep{qian2025toolrl}.

Beyond tool invocation itself, recent RL methods for language models have explored scalable policy optimization algorithms that compare multiple sampled responses for the same input. Proximal Policy Optimization provides a widely used foundation for stable policy-gradient training \citep{schulman2017proximal}, while Group Relative Policy Optimization removes the need for a learned value model by estimating advantages from relative rewards within a group of rollouts \citep{shao2024deepseekmath}. This line of work provides a practical optimization framework for training language agents with rule-based or task-specific rewards.
\section{Method}
\begin{figure}[t]
  \centering
  \includegraphics[width=\linewidth]{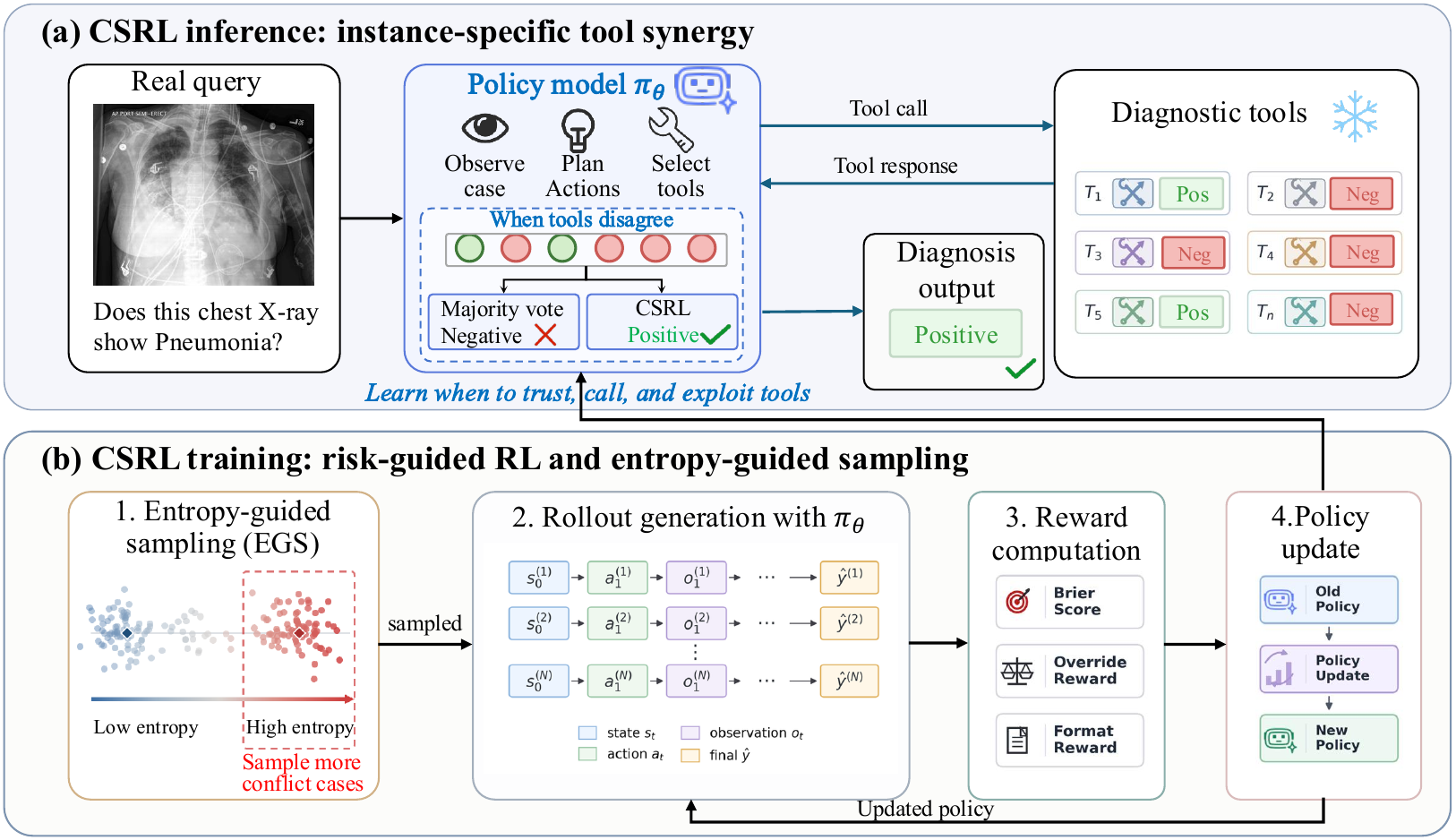}
  \caption{
  Overview of CSRL.
  (a) CSRL inference performs instance-specific agentic tool use.
  (b) CSRL training combines entropy-guided sampling with risk-guided reinforcement learning.
  }
  \label{fig:overview}
  \vskip -1pt
\end{figure}
\subsection{Preliminary}


\noindent\textbf{Problem Formulation. }Given a diagnostic sample $(x,q,y)\sim\mathcal{D}$, where 
$x\in\mathcal{X}$ is the input image, $q\in\mathcal{Q}$ is the target pathology query, and $y\in\mathcal{Y}$ is the ground-truth
label, we consider a fixed pool of $K$ pretrained diagnostic tools
$\mathcal{T}=\{T_1,\ldots,T_K\}$. For a given case, the policy may invoke a subset of tools from $\mathcal{T}$.
For any invoked tool $T_k$, the final output of each tool can be represented as 
$p_k=T_k(x,q)\in\mathcal{P}$. 
Due to heterogeneity in model architectures and training data, these outputs may exhibit non-overlapping error patterns.

To formalize this complementarity, we first define
the best single-tool risk $\mathcal{R}_{\mathrm{single}}^\star$ as
\begin{equation}
\begin{aligned}
\mathcal{R}_{\mathrm{single}}^\star
&=
\min_{k\in[K]}
\mathbb{E}_{(x,q,y)\sim\mathcal{D}}
\left[
\ell(p_k,y)
\right],
\end{aligned}
\label{eq:risks}
\end{equation}
Where
$\ell\in \mathcal{L}:\mathcal{P}\times\mathcal{Y}\rightarrow\mathbf{R} $ is a  diagnostic loss, and $\mathcal{L}$ is the space of loss functions. As we discussed above, a single best tool is imperfect in real-world scenarios. Thus, we consider discovering case-level tool selection. Specifically, given
$\Pi_{\mathcal{T}}$ be the space of tool integration policy $\pi:\mathcal{X}\times\mathcal{Q}\times\mathcal{P}^K\rightarrow \mathbf{R}$, the Oracle risk can be defined as
\begin{equation}
\begin{aligned}
\mathcal{R}_{\mathrm{oracle}}^\star
&=
\inf_{\pi\in\Pi_{\Theta}}
\mathbb{E}_{(x,q,y)\sim\mathcal{D}}
\left[
\ell\!\left(\pi(x,q,\mathbf{p}),y\right)
\right].
\end{aligned}
\label{eq:risks_oracl}
\end{equation}
Since the single tool strategy $\pi_k(x,q,\mathbf{p})=p_k$ is a special case in $\Pi_{\mathcal{T}}$, it is obvious that we can define the Single-Oracle risk gap as
\begin{equation}
\Delta_{\mathcal{T}}
=
\mathcal{R}_{\mathrm{single}}^\star
-
\mathcal{R}_{\mathrm{oracle}}^\star
\geq 0.
\label{eq:synergy_gap}
\end{equation}
A positive $\Delta_{\mathcal{T}}$ indicates that the strongest individual tool is not the performance
limit of the tool pool. As shown in the Introduction, performance gains ($\Delta_\mathcal{T}$) are significantly large, which supports our assumption.

To optimize the model toward oracle-level performance, we aim to learn a parameterized tool-use policy $\pi_\theta\in\Pi_{\mathcal{T}}$.
Given $\hat p_\theta=\pi_\theta(x,q,\mathbf{p})$,
we aim to learn a policy satisfying
\begin{equation}
\mathcal{R}_{\mathrm{oracle}}^\star
\leq
\mathcal{R}(\pi_\theta)
=
\mathbb{E}_{(x,q,y)\sim\mathcal{D}}
\left[
\ell(\hat p_\theta,y)
\right]
<
\mathcal{R}_{\mathrm{single}}^\star .
\label{eq:csrl_target}
\end{equation}

To this end, as provided in Figure~\ref{fig:overview}, we propose CSRL, a Collaborative Synergy Reinforcement Learning framework built upon GRPO~\cite{shao2024deepseekmath}, which introduces two components that align the optimization procedure with the target in Eq.~(4). First, we design trajectory-level rewards for GRPO that encourage valid tool use, minimize probabilistic diagnostic risk, and promote disagreement-aware synergy recovery, thereby directly pushing down \(\mathcal R(\pi_\theta)\). Second, we use Entropy-Guided Sampling to increase the training mass of high-disagreement cases, where the gap between single-tool performance and case-specific integration is most recoverable. 
\paragraph{Tool-Calling Rollout.} 
Rollouts $\{\tau_i\}$ define the interaction protocol between the policy $\pi_\theta$ and the tool-augmented environment: the policy $\pi_\theta$ uses \ToolCallTok...\EndToolCallTok to invoke diagnostic tools, receives tool feedback wrapped by \ToolRespTok...\EndToolRespTok, and terminates the trajectory by producing a final prediction inside \AnswerTok...\EndAnswerTok. The detailed procedure is provided in Appendix~\ref{app:rollout}.
\paragraph{GRPO Optimization.} We optimize $\pi_\theta$ with \textbf{Group Relative Policy Optimization}
(GRPO)~\citep{shao2024deepseekmath}, using EGS-stratified mini-batches
(Sec.~\ref{sec:egs}) and the trajectory-level reward $R_{\mathrm{overall}}$
from Sec.~\ref{sec:reward}. For each prompt $(x,q)$, GRPO samples $G$ rollouts
$\{\tau_i\}_{i=1}^G$ from $\pi_{\theta_{\mathrm{old}}}$ and assigns each rollout
a group-relative advantage
\begin{equation}
\hat A_i \;=\;
\frac{
R_{\mathrm{overall}}(\tau_i)\,-\,\mathrm{mean}\!\left(\{R_{\mathrm{overall}}(\tau_j)\}_{j=1}^{G}\right)
}{
\mathrm{std}\!\left(\{R_{\mathrm{overall}}(\tau_j)\}_{j=1}^{G}\right)
},
\label{eq:grpo_advantage}
\end{equation}
which is shared across all policy-generated tokens in $\tau_i$. The GRPO
objective combines a PPO-style clipped surrogate with a token-level KL penalty
to a frozen reference policy $\pi_{\mathrm{ref}}$:
\begin{equation}
\mathcal{J}_{\mathrm{GRPO}}(\theta)
\;=\;
\mathbb{E}\!\left[
\frac{1}{G}\sum_{i=1}^{G}\frac{1}{|\tau_i|}\sum_{t=1}^{|\tau_i|}
\Big\{
\min\!\big(
\rho_{i,t}\hat A_i,\;
\mathrm{clip}(\rho_{i,t},\,1\!-\!\epsilon,\,1\!+\!\epsilon)\,\hat A_i
\big)
-\beta\,\mathbb{D}_{\mathrm{KL}}\!\big[\pi_\theta\Vert\pi_{\mathrm{ref}}\big]
\Big\}
\right],
\label{eq:grpo_objective}
\end{equation}
with importance ratio
\begin{equation}
\rho_{i,t}
=
\frac{\pi_\theta(a_{i,t}\!\mid\!s_{i,t})}
{\pi_{\theta_{\mathrm{old}}}(a_{i,t}\!\mid\!s_{i,t})}.
\label{eq:importance_ratio}
\end{equation}

\subsection{Risk-guided Reward Design}
\label{sec:reward}
To train the policy $\pi_\theta$ with GRPO, we assign each rollout a rule-based
reward composed of three terms that capture format validity, probabilistic
correctness, and case-specific majority override: 
\begin{equation}
R_{\mathrm{overall}}
=
R_{\mathrm{brier}}
+
R_{\mathrm{override}}+R_{\mathrm{format}} .
\label{eq:reward}
\end{equation}

To encourage a low $\mathcal{R}(\pi_\theta)$ small, we first introduce the Brier reward as follows: 
\begin{equation}
R_{\mathrm{brier}}(\hat{p}_\theta, y)
=
\begin{cases}
1-(\hat p_\theta-y)^2, &
\begin{array}{l}
\text{if } \hat p_\theta\in[0,1] \text{ is parseable and }  n_{\mathrm{valid}}>0,
\end{array}
\\[0.5em]
0, & \text{otherwise}.
\end{cases}
\label{eq:brier_reward}
\end{equation}
Where, $n_{\mathrm{valid}}$ denote the number of valid tool responses observed in
the rollout. The Brier reward evaluates probabilistic correctness only when the
rollout uses valid tool evidence. The following proposition characterizes the role of $R_{\mathrm{brier}}$
\begin{proposition}
Assume that the prediction risk is defined with the Brier loss, i.e.,
\[
\mathcal{R}(\pi_\theta)
=
\mathbb{E}_{(x,q,y)\sim \mathcal{D}}
\left[
(\hat{p}_\theta - y)^2
\right].
\]
Then maximizing the Brier reward is equivalent to minimizing the prediction risk:
\[
\arg\max_\theta \mathbb{E}
\left[
R_{\mathrm{brier}}(\hat{p}_\theta, y)
\right]
=
\arg\min_\theta \mathcal{R}(\pi_\theta).
\]
\label{prop1}
\end{proposition}
Proposition \ref{prop1} illustrates that maximizing reward $R_{\mathrm{brier}}$ is aligned with minimizing the risk $\mathcal{R}(\pi_\theta)$. $R_{\mathrm{brier}}$ provides a proper probabilistic training signal for the final diagnostic score, but it is agnostic to the structure of tool disagreement. Such disagreement is central to the synergy gap, as it indicates that different tools succeed and fail on different cases, leaving room for a case-specific policy to recover information that no fixed tool can consistently capture. 
\par
To address the above issue, we introduce the override reward in the following.
\begin{equation}
R_{\mathrm{override}}(\hat p_\theta,  m, y)
=
\begin{cases}
\alpha, & \text{if } \ell(\hat p_\theta,y)< s \text{ and } \ell(\hat p_m, y)\geq s,\\
-\alpha, & \text{if } \ell(\hat p_\theta, y)\geq s \text{ and } \ell(\hat p_m,y)<s,\\
0, & \text{otherwise}.
\end{cases}
\label{eq:override_reward}
\end{equation}
Where $\hat p_m$ denotes the response of the majority vote among the valid invoked tools, $s$ is a threshold used to determine whether a prediction is consistent with the ground-truth label $y$ based on the loss $\ell$, e.g., $s=0.5$ in classification tasks, and $\alpha$ is the reward value which is set as 0.1 in our experiments. The following proposition demonstrates how $R_{\mathrm{override}}$ enlarges the risk gap.
\begin{proposition}
Considering the risk
\(\mathcal{R}(\pi_\theta)=\Pr(\hat p_\theta \neq y)\)
and the risk gap
\(\mathcal{G}(\pi_\theta)=\mathcal{R}^{\star}_{\mathrm{single}}-\mathcal{R}(\pi_\theta)\),
maximizing \(R_{\mathrm{override}}(\hat p_\theta,m,y)\) is aligned with maximizing
\(\mathcal{G}(\pi_\theta)\).
\label{prop2}
\end{proposition}
Proposition~\ref{prop2} shows that the override reward is aligned with our objective of increasing $\mathcal{G}(\pi_\theta)$, thereby improving the policy beyond the best single-tool baseline. \par
In addition, to encourage the policy to follow the required structured
output schema, we introduce a format reward  as follows:
\begin{equation}
R_{\mathrm{format}}
=
\begin{cases}
\alpha, &
\begin{array}{l}
\text{if all required fields appear in the correct order} \\
\text{and the terminal answer is parseable into } [0,1],
\end{array}
\\[0.5em]
0, & \text{otherwise}.
\end{cases}
\label{eq:format_reward}
\end{equation}
In our experiments, we set $\alpha=0.1$ in $R_{\mathrm{format}}$. Although we present the reward design under a classification setting in this section, it can be readily adapted to regression tasks, which we discuss in the Experiments section.

\subsection{Entropy-Guided Sampling}
\label{sec:egs}
The risk gap \(\Delta_{\mathcal T}\) is most visible in cases where the tool pool provides heterogeneous evidence. Under uniform sampling from the natural data distribution, these high-disagreement cases may be underrepresented, causing most GRPO groups to be dominated by consensus cases with weak signals for learning case-specific tool synergy. We therefore introduce Entropy-Guided Sampling (EGS), a static \emph{per-batch} stratified sampler that reweights mini-batch construction toward high-entropy regions of the tool-evidence space. By increasing the frequency of disagreement cases during training, EGS provides denser policy-gradient signal for learning when and how to exploit complementary tools.\par
For each training case $(x,q,y)$, we record the number of correctly classifying tools:
\begin{equation}
n_c(x,q) = \sum_{k=1}^{K} \mathbb{I}\{\ell(\hat p_k,y)<s\},
\label{eq:n_correct}
\end{equation}
where $s$ is the threshold discussed in Eq.~\ref{eq:override_reward}. We then measure the correctness-agreement entropy by
\begin{equation}
H(\bar n) = -\bar n\log\bar n - (1-\bar n)\log(1-\bar n), \qquad \bar n=\frac{n_c}{K}.
\label{eq:vote_entropy}
\end{equation}
This entropy quantifies case-level tool disagreement, reaching its maximum at \(n_c=K/2\) and vanishing at \(n_c\in\{0,K\}\). The useful learning signal for case-specific synergy is concentrated in the conflict region \(1\leq n_c\leq K-1\), where correct and incorrect tools coexist. In contrast, consensus cases provide limited synergy signal: \(n_c=0\) suggests either label noise or cases outside the capability of the tool pool, whereas \(n_c=K\) indicates that all tools are already correct and case-specific aggregation cannot further improve the prediction. \par
We therefore stratify the training set according to the entropy-induced disagreement level. Let \(\{\mathcal{S}_j\}_{j=1}^{K'}\) denote the resulting subsets of all data, ordered by increasing entropy. During mini-batch construction, each subset \(\mathcal{S}_j\) is assigned a sampling probability \(\rho_j\) proportional to its entropy level. Since the entropy of $S_1$ is obviously 0, we assign a small probability to this subset exclusively. Consequently, high-entropy subsets, which contain stronger tool disagreement and richer synergy signals, are sampled more frequently than low-entropy consensus subsets. This entropy-proportional sampling scheme biases training toward cases where case-specific tool integration is most beneficial, while preserving a nonzero sampling probability for all subsets to maintain calibration on easy or consensus cases.

\section{Experiments}
\subsection{EXPERIMENTAL SETUP}
\textbf{Datasets.}
We conduct experiments on two tasks and seven datasets on the chest X-ray (CXR) modality. For
\emph{binary pathology classification} (Task~1), CSRL is trained on the
official training splits of CheXpert~\cite{irvin2019chexpert}
and MIMIC-CXR~\cite{johnson2019mimic}, and evaluated on six test
benchmarks: their held-out test splits, three stratified $1{,}000$-image
out-of-domain samples from NIH ChestX-ray14~\cite{wang2017chestx},
RSNA Pneumonia Detection~\cite{shih2019augmenting}, and
VinDr-CXR~\cite{nguyen2022vindr}, plus
NIH-Google~\cite{majkowska2020chest} ($1{,}962$
radiologist-adjudicated images) as a clean-label check against NIH NLP
label noise.  For \emph{ medical VQA generalization} (Task~2), the policy is trained on a stratified $5{,}000$-sample subset of the \textit{choose\_train} split of
MIMIC-Ext-MIMIC-CXR-VQA~\cite{bae2023ehrxqa} under the M4CXR non-empty-GT setting, and evaluated on the gold $1{,}523$-question test split and its $933$-question non-empty-GT subset~\cite{park2025m4cxr}.

\textbf{Policy Model and Tool Pools.} We use Qwen2.5-VL-7B-Instruct~\cite{Bai2025Qwen25VLTR} as the policy
model for reinforcement learning. We assemble two task-specific tool
pools chosen for diversity in training data and architecture. For Task~1
(\emph{binary pathology classification}), the tool pool includes six
specialist tools, namely CheXagent~\cite{chen2024chexagent}, CheXzero~\cite{tiu2022expert}, MedKLIP~\cite{wu2023medklip}, and three DenseNet-121 backbones from
torchxrayvision~\cite{cohen2020torchxrayvision}. For Task~2
(\emph{VQA generalization}), the tool pool includes three medical
vision-language models, namely CheXagent-8B~\cite{chen2024chexagent}, RadVLM-7B~\cite{deperrois2025radvlm}, and
MedGemma-4B-IT~\cite{sellergren2025medgemma}.

\textbf{Implementation Details.} We implement CSRL with the Qwen2.5-VL-7B-Instruct~\cite{Bai2025Qwen25VLTR}
backbone. The policy is
optimized using the Group Relative Policy Optimization (GRPO)
algorithm~\cite{shao2024deepseekmath,guo2025deepseek} for
$\sim\!75$ iterations on 8 NVIDIA H200 GPUs, built on
\textsc{VeRL}~\cite{sheng2025hybridflow} with FSDP training and asynchronous
vLLM~\cite{kwon2023efficient} rollouts. During training, we use a batch size
of $256$ prompts; for each prompt, the agent samples $16$ trajectories with
up to $4$ reasoning turns and a maximum of $6$ parallel tool calls per turn. Key hyperparameters are AdamW with learning rate $1{\times}10^{-6}$, PPO clip
ratio $0.2$, low-variance KL regularization with coefficient $10^{-3}$ against
a frozen reference policy, and an entropy bonus of $10^{-3}$.

\textbf{Baselines and Comparisons.} We compare CSRL with a broad set of baselines spanning four categories: open-source general vision-language models (VLMs), open-source medical MLLMs, closed-source general MLLMs, and specialist CXR diagnostic tools. To further assess whether the gains come from case-specific synergy rather than simply using multiple tools, we include several tool-combination baselines, including logistic regression, majority voting, and best-tool-per-pathology selection.

\subsection{Experimental Analysis}

\begin{table*}[t]
\centering
\caption{Performance comparison for binary disease diagnosis across six chest X-ray benchmarks. We report macro accuracy (Acc) and macro F1. Out-of-domain (OOD) test sets are marked with $\diamond$. Best results are highlighted in \textbf{bold}, and second-best results are \underline{underlined}.}
\label{tab:main_results}
\scriptsize
\setlength{\tabcolsep}{2.6pt}
\renewcommand{\arraystretch}{1.08}
\resizebox{0.98\textwidth}{!}{%
\begin{tabular}{@{}l*{7}{cc}@{}}
\toprule
\multicolumn{1}{c}{\multirow{2}{*}{\textbf{Methods}}}
& \multicolumn{2}{c}{\textbf{CheXpert}}
& \multicolumn{2}{c}{\textbf{MIMIC-CXR}}
& \multicolumn{2}{c}{\textbf{ChestX-ray14}$^{\diamond}$}
& \multicolumn{2}{c}{\textbf{VinDr-CXR}$^{\diamond}$}
& \multicolumn{2}{c}{\textbf{NIH-Google}$^{\diamond}$}
& \multicolumn{2}{c}{\textbf{RSNA}$^{\diamond}$}
& \multicolumn{2}{c}{\textbf{Average}} \\
\cmidrule(lr){2-3}
\cmidrule(lr){4-5}
\cmidrule(lr){6-7}
\cmidrule(lr){8-9}
\cmidrule(lr){10-11}
\cmidrule(lr){12-13}
\cmidrule(lr){14-15}
& Acc$\uparrow$ & F1$\uparrow$
& Acc$\uparrow$ & F1$\uparrow$
& Acc$\uparrow$ & F1$\uparrow$
& Acc$\uparrow$ & F1$\uparrow$
& Acc$\uparrow$ & F1$\uparrow$
& Acc$\uparrow$ & F1$\uparrow$
& Acc$\uparrow$ & F1$\uparrow$ \\
\midrule

\rowcolor{groupgray}
\multicolumn{15}{c}{\textbf{Open-source General VLMs}} \\
Qwen2.5-VL-7B-Instruct~\cite{Bai2025Qwen25VLTR}
& 0.731 & 0.181
& 0.539 & 0.371
& \second{0.722} & 0.155
& 0.817 & 0.121
& 0.793 & 0.378
& 0.771 & 0.336
& 0.729 & 0.257 \\
InternVL3.5-38B-Instruct~\cite{wang2025internvl3}
& 0.744 & 0.451
& 0.790 & 0.710
& 0.585 & 0.228
& 0.746 & 0.340
& \second{0.879} & 0.478
& \best{0.790} & \second{0.628}
& 0.756 & 0.472 \\
Qwen2.5-VL-72B-Instruct~\cite{Bai2025Qwen25VLTR}
& 0.382 & 0.309
& 0.713 & 0.703
& 0.375 & 0.174
& 0.435 & 0.100
& 0.785 & 0.429
& 0.425 & 0.429
& 0.519 & 0.357 \\

\midrule
\rowcolor{groupgray}
\multicolumn{15}{c}{\textbf{Open-source Medical MLLMs}} \\
HuatuoGPT-Vision-7B~\cite{chen2024towards}
& 0.699 & 0.397
& 0.698 & 0.667
& 0.512 & 0.212
& 0.894 & 0.282
& 0.847 & 0.535
& 0.726 & 0.580
& 0.729 & 0.446 \\
Lingshu-7B~\cite{xu2025lingshu}
& 0.534 & 0.386
& 0.711 & 0.751
& 0.532 & 0.259
& 0.770 & 0.248
& 0.723 & 0.597
& 0.706 & 0.575
& 0.663 & 0.469 \\
MedGemma-27B-IT~\cite{sellergren2025medgemma}
& 0.722 & 0.454
& 0.721 & 0.733
& 0.620 & \second{0.280}
& 0.902 & 0.403
& 0.777 & 0.615
& 0.742 & 0.585
& 0.747 & 0.512 \\
Hulu-Med-Flash-Preview-27B~\cite{jiang2025hulu}
& 0.728 & 0.252
& 0.767 & 0.767
& \best{0.727} & 0.277
& 0.868 & 0.411
& 0.854 & 0.666
& \best{0.790} & 0.291
& \second{0.789} & 0.444 \\

\midrule
\rowcolor{groupgray}
\multicolumn{15}{c}{\textbf{Closed-source General MLLMs}} \\
Claude-Sonnet-4.6
& 0.473 & 0.337
& 0.709 & 0.708
& 0.330 & 0.192
& 0.543 & 0.129
& 0.754 & 0.523
& 0.465 & 0.433
& 0.546 & 0.387 \\
GPT-4o
& 0.582 & 0.355
& 0.699 & 0.672
& 0.499 & 0.204
& 0.851 & 0.201
& 0.845 & 0.509
& 0.687 & 0.556
& 0.694 & 0.416 \\
Gemini-3.1-Pro-Preview
& 0.685 & 0.436
& 0.769 & 0.760
& 0.535 & 0.254
& 0.930 & \second{0.455}
& 0.849 & 0.670
& 0.731 & 0.602
& 0.750 & \second{0.530} \\

\midrule
\rowcolor{groupgray}
\multicolumn{15}{c}{\textbf{Specialist CXR Tools}} \\
DenseNet-CheX~\cite{cohen2022torchxrayvision}
& 0.682 & 0.433
& 0.704 & 0.711
& 0.517 & 0.225
& 0.916 & 0.271
& 0.702 & 0.522
& 0.686 & 0.505
& 0.701 & 0.445 \\
DenseNet-MIMIC~\cite{cohen2022torchxrayvision}
& 0.657 & 0.413
& 0.701 & 0.686
& 0.559 & 0.208
& 0.787 & 0.179
& 0.774 & 0.491
& 0.703 & 0.461
& 0.697 & 0.406 \\
DenseNet-NIH~\cite{cohen2022torchxrayvision}
& 0.607 & 0.379
& 0.665 & 0.688
& 0.440 & 0.218
& 0.732 & 0.157
& 0.611 & 0.500
& 0.689 & 0.498
& 0.624 & 0.407 \\
CheXzero~\cite{tiu2022expert}
& 0.763 & 0.489
& 0.746 & 0.756
& 0.603 & 0.274
& 0.904 & 0.362
& 0.781 & 0.614
& 0.683 & 0.563
& 0.747 & 0.510 \\
MedKLIP~\cite{wu2023medklip}
& 0.713 & 0.464
& 0.737 & 0.762
& 0.580 & 0.261
& 0.897 & 0.383
& 0.677 & 0.565
& 0.714 & 0.589
& 0.720 & 0.504 \\
CheXagent-8B~\cite{chen2024vision}
& 0.786 & 0.498
& \second{0.808} & 0.802
& 0.507 & 0.263
& 0.502 & 0.225
& 0.746 & 0.614
& 0.492 & 0.451
& 0.640 & 0.476 \\

\midrule
\rowcolor{groupgray}
\multicolumn{15}{c}{\textbf{Tool Combination Baselines}} \\
Logistic Regression
& 0.784 & 0.510
& 0.791 & 0.805
& 0.586 & 0.256
& 0.821 & 0.369
& 0.853 & 0.649
& 0.459 & 0.446
& 0.716 & 0.506 \\
Majority Voting ($\mathrm{tie}\rightarrow\mathrm{neg}$)
& 0.766 & 0.490
& 0.774 & 0.762
& 0.608 & 0.270
& \second{0.934} & 0.409
& 0.830 & 0.634
& 0.745 & 0.602
& 0.776 & 0.528 \\
Majority Voting ($\mathrm{tie}\rightarrow\mathrm{pos}$)
& 0.696 & 0.467
& 0.779 & 0.779
& 0.494 & 0.249
& 0.875 & 0.333
& 0.731 & 0.586
& 0.682 & 0.569
& 0.709 & 0.497 \\
Best Tool per Pathology
& \second{0.801} & \second{0.516}
& 0.804 & \best{0.812}
& 0.617 & 0.278
& 0.822 & 0.374
& 0.866 & \second{0.674}
& 0.492 & 0.451
& 0.734 & 0.518 \\

\midrule
\rowcolor{oursblue}
\textbf{Ours (CSRL)}
& \best{0.859} & \best{0.563}
& \best{0.810} & \second{0.810}
& 0.664 & \best{0.301}
& \best{0.935} & \best{0.482}
& \best{0.889} & \best{0.682}
& \second{0.772} & \best{0.630}
& \best{0.822} & \best{0.578} \\

\bottomrule
\end{tabular}%
}
\vspace{-0.5em}
\end{table*}

\begin{table}[ht!]
\centering
\scriptsize
\setlength{\tabcolsep}{3.0pt}
\renewcommand{\arraystretch}{1.15}
\providecommand{\allabldeltasize}{\fontsize{5.8pt}{6.2pt}\selectfont}
\providecommand{\allabldec}[1]{\,\raisebox{-0.20ex}{\textcolor{red}{{\allabldeltasize (#1\%)}}}}
\providecommand{\allablinc}[1]{\,\raisebox{-0.20ex}{\textcolor{green!50!black}{{\allabldeltasize (#1\%)}}}}
\providecommand{\allablsame}[1]{\,\raisebox{-0.20ex}{\textcolor{gray}{{\allabldeltasize (#1\%)}}}}

\caption{
Ablation study on three CXR benchmarks. We report macro accuracy (Acc) and macro F1. $\diamond$: OOD test sets. ``w. $n$ tool(s)'': allow at most $n$ diagnostic tools per rollout. All models are trained on CheXpert and MIMIC-CXR, and all changes are measured relative to CSRL.
}
\label{tab:ablation_csrl_all}

\vspace{5pt}

\resizebox{\linewidth}{!}{
\begin{tabular}{l cc cc cc cc}
\toprule
\multirow{2}{*}{\textbf{Method}}
& \multicolumn{2}{c}{\textbf{CheXpert}}
& \multicolumn{2}{c}{\textbf{MIMIC-CXR}}
& \multicolumn{2}{c}{\textbf{ChestX-ray14}$^\diamond$}
& \multicolumn{2}{c}{\textbf{Average}} \\
\cmidrule(lr){2-3}
\cmidrule(lr){4-5}
\cmidrule(lr){6-7}
\cmidrule(lr){8-9}
& \textbf{Acc$\uparrow$} & \textbf{F1$\uparrow$}
& \textbf{Acc$\uparrow$} & \textbf{F1$\uparrow$}
& \textbf{Acc$\uparrow$} & \textbf{F1$\uparrow$}
& \textbf{Acc$\uparrow$} & \textbf{F1$\uparrow$} \\
\midrule

\multicolumn{9}{l}{\emph{Ablation for Policy Model}} \\
\quad w. Qwen3-VL-2B-Instruct as policy
& 0.783\allabldec{-7.6} & 0.504\allabldec{-5.9}
& 0.810\allablsame{0.0} & 0.807\allabldec{-0.3}
& 0.577\allabldec{-8.7} & 0.275\allabldec{-2.7}
& 0.723\allabldec{-5.5} & 0.529\allabldec{-3.0} \\

\quad w. Qwen2.5-VL-3B-Instruct as policy
& 0.812\allabldec{-4.7} & 0.535\allabldec{-2.8}
& 0.797\allabldec{-1.3} & 0.807\allabldec{-0.4}
& 0.626\allabldec{-3.8} & 0.291\allabldec{-1.1}
& 0.745\allabldec{-3.3} & 0.544\allabldec{-1.4} \\

\midrule
\multicolumn{9}{l}{\emph{Ablation for Training Data Scale}} \\
\quad w. 10\% training data
& 0.830\allabldec{-2.9} & 0.525\allabldec{-3.8}
& 0.787\allabldec{-2.3} & 0.752\allabldec{-5.8}
& 0.631\allabldec{-3.3} & 0.279\allabldec{-2.2}
& 0.749\allabldec{-2.8} & 0.519\allabldec{-4.0} \\

\quad w. 50\% training data
& 0.854\allabldec{-0.5} & 0.560\allabldec{-0.3}
& 0.787\allabldec{-2.3} & 0.776\allabldec{-3.4}
& 0.665\allablinc{+0.1} & 0.296\allabldec{-0.5}
& 0.769\allabldec{-0.9} & 0.544\allabldec{-1.4} \\

\midrule
\multicolumn{9}{l}{\emph{Ablation for Training Dataset Composition}} \\
\quad w. CheXpert-only training
& 0.844\allabldec{-1.5} & 0.552\allabldec{-1.1}
& 0.802\allabldec{-0.8} & 0.781\allabldec{-3.0}
& 0.649\allabldec{-1.5} & 0.295\allabldec{-0.6}
& 0.765\allabldec{-1.3} & 0.543\allabldec{-1.6} \\

\quad w. extra ChestX-ray14 training data$^\dagger$
& 0.845\allabldec{-1.4} & 0.553\allabldec{-1.0}
& 0.805\allabldec{-0.5} & 0.802\allabldec{-0.9}
& 0.638\allabldec{-2.6} & 0.301\allablsame{0.0}
& 0.763\allabldec{-1.5} & 0.552\allabldec{-0.6} \\

\midrule
\multicolumn{9}{l}{\emph{Ablation for Sampling and Reward Design}} \\
\quad w/o Entropy-Guided Sampling
& 0.828\allabldec{-3.1} & 0.519\allabldec{-4.5}
& 0.785\allabldec{-2.5} & 0.745\allabldec{-6.5}
& 0.654\allabldec{-1.0} & 0.275\allabldec{-2.6}
& 0.756\allabldec{-2.2} & 0.513\allabldec{-4.5} \\

\quad w/o override reward
& 0.848\allabldec{-1.1} & 0.544\allabldec{-1.9}
& 0.793\allabldec{-1.7} & 0.779\allabldec{-3.1}
& 0.664\allablsame{0.0} & 0.294\allabldec{-0.7}
& 0.768\allabldec{-0.9} & 0.539\allabldec{-1.9} \\

\midrule
\multicolumn{9}{l}{\emph{Ablation for Tool-Call Budget}} \\
\quad w. 1 tool
& 0.786\allabldec{-7.3} & 0.500\allabldec{-6.3}
& 0.816\allablinc{+0.6} & 0.809\allabldec{-0.2}
& 0.507\allabldec{-15.7} & 0.263\allabldec{-3.8}
& 0.703\allabldec{-7.5} & 0.524\allabldec{-3.4} \\

\quad w. 2 tools
& 0.792\allabldec{-6.8} & 0.504\allabldec{-5.9}
& 0.817\allablinc{+0.7} & 0.809\allabldec{-0.1}
& 0.528\allabldec{-13.6} & 0.267\allabldec{-3.4}
& 0.712\allabldec{-6.6} & 0.527\allabldec{-3.2} \\

\quad w. 3 tools
& 0.859\allablsame{+0.0} & 0.558\allabldec{-0.5}
& 0.779\allabldec{-3.1} & 0.772\allabldec{-3.8}
& 0.690\allablinc{+2.6} & 0.297\allabldec{-0.4}
& 0.776\allabldec{-0.2} & 0.543\allabldec{-1.6} \\

\quad w. 4 tools
& 0.849\allabldec{-1.0} & 0.557\allabldec{-0.7}
& 0.797\allabldec{-1.3} & 0.788\allabldec{-2.2}
& 0.672\allablinc{+0.8} & 0.295\allabldec{-0.6}
& 0.773\allabldec{-0.5} & 0.547\allabldec{-1.2} \\

\quad w. 5 tools
& 0.846\allabldec{-1.3} & 0.550\allabldec{-1.3}
& 0.788\allabldec{-2.1} & 0.785\allabldec{-2.5}
& 0.684\allablinc{+2.0} & 0.297\allabldec{-0.4}
& 0.773\allabldec{-0.5} & 0.544\allabldec{-1.4} \\

\midrule
\rowcolor{oursblue}
\textbf{CSRL (ours)}
& 0.859 & 0.563
& 0.810 & 0.810
& 0.664 & 0.301
& 0.778 & 0.558 \\

\bottomrule
\end{tabular}
}
\end{table}

To systematically evaluate the effectiveness of our proposed approach, we organize the experiment as follows.

\textbf{Learning tool synergy yields robust gains across in-domain and out-of-domain datasets.} 
Our 
\begin{wraptable}{r}{0.5\linewidth}
\vspace{-8pt}
\centering
\scriptsize
\setlength{\tabcolsep}{4.2pt}
\renewcommand{\arraystretch}{1.12}
\caption{
Generalization performance on the VQA task.
We report F1, exact match (EM), precision (P), recall (R), and their average score.
Best results are highlighted in \textbf{bold}.
}
\label{tab:vqa_generalization}
\vspace{3pt}

\resizebox{\linewidth}{!}{
\begin{tabular}{lccccc}
\toprule
\textbf{Model} & \textbf{F1$\uparrow$} & \textbf{EM$\uparrow$} & \textbf{P$\uparrow$} & \textbf{R$\uparrow$} & \textbf{Avg.$\uparrow$} \\
\midrule
Qwen2.5-VL-7B-Instruct
& 0.417 & 0.268 & 0.431 & 0.439 & 0.389 \\
MedGemma-4B-IT
& 0.691 & 0.509 & 0.712 & 0.715 & 0.657 \\
RadVLM-7B
& 0.653 & 0.528 & 0.698 & 0.638 & 0.629 \\
CheXagent-8B
& 0.682 & 0.545 & 0.750 & 0.647 & 0.656 \\
\midrule
\textbf{CSRL (ours)}
& \textbf{0.737} & \textbf{0.564} & \textbf{0.758} & \textbf{0.758} & \textbf{0.704} \\
\bottomrule
\end{tabular}
}
\end{wraptable}
proposed CSRL achieves the best overall performance among all compared methods (see Table~\ref{tab:main_results}), obtaining an average Acc of 0.822 and an average F1 of 0.578 across the six CXR benchmarks, outperforming the strongest prior method by 3.3\% in accuracy and 4.8\% in F1 score. Compared with the strongest individual tool, CheXzero, CSRL improves the average performance by \textbf{7.5\%} in Acc and \textbf{6.8\%} in F1, showing that explicitly learning tool synergy is more effective than relying on any single tool. 
This advantage is also consistent across datasets: CSRL almost achieves the best or the second-best F1 score and Acc on all six benchmarks. Notably, CSRL performs strongly on both in-domain and out-of-domain test sets, including ChestX-ray14, VinDr-CXR, NIH-Google, and RSNA, suggesting that learned tool synergy helps the model better exploit complementary diagnostic capabilities under distribution shifts. In contrast, individual VLMs, specialist CXR tools, and simple combination baselines often perform well only on specific datasets or metrics, but their average performance remains lower. These results demonstrate that CSRL can learn when and how different tools complement each other, leading to more robust and accurate binary disease diagnosis than the best individual tool.

Beyond binary disease diagnosis, Table~\ref{tab:vqa_generalization} further evaluates the generalization ability of CSRL on the VQA task. CSRL consistently achieves the best performance across all metrics. Compared with the strongest baseline, MedGemma-4B-IT, CSRL improves the average score from 0.657 to 0.704, corresponding to a gain of \textbf{4.7\%}, indicating that learning tool synergy not only improves classification performance but also enhances the model's ability to recover more complete and accurate answers in a different task setting.

\textbf{Trade-off between tool-set size, performance, and cost.}
Figure~\ref{fig:scalability-ablation} reveals a clear trade-off between tool-set size, performance, and computational cost. Increasing the number of tools generally improves performance, but the marginal gain diminishes as the tool set becomes larger. As shown in the tool scalability ablation (Figure~\ref{fig:tool_scalability} ), expanding the tool set from T1 to T3 brings a substantial improvement, with average macro Acc increasing from 70.29\% to 77.60\% and average macro F1 increasing from 52.40\% to 54.25\%. However, further increasing the tool set from T3 to T6 yields only limited additional gains, with Acc fluctuating around 77\%--78\% and F1 improving more moderately to 55.82\%. This suggests that a small number of complementary tools can already capture most of the useful inter-tool synergy, while adding more tools may introduce redundancy and higher inference cost.

The results also show that performance is affected not only by the number of tools, but also by the capacity of the policy model (Figure ~\ref{fig:policy_model_scalability} ) and the quality of training data (Figures ~\ref{fig:data_scalability} and ~\ref{fig:dataset_scalability}). Larger policy models consistently improve performance, from 72.31\% Acc and 52.85\% F1 with the 2B model to 77.76\% Acc and 55.82\% F1 with the 7B model, indicating that stronger policy models are better at learning when different tools should be trusted or combined. Similarly, increasing the training data scale improves both Acc and F1, while using a more diverse training composition further strengthens generalization. Overall, these results indicate that the best performance-cost balance is achieved by selecting a compact but complementary tool set and pairing it with a sufficiently capable policy model, rather than naively increasing the number of tools.

\subsection{Ablation Study}
Table~\ref{tab:ablation_csrl_all} presents a comprehensive ablation study of CSRL. Overall, the full CSRL model achieves the best average performance, with an average Acc of 0.778 and an average F1 of 0.558, demonstrating the effectiveness of each major design component. First, reducing the policy model scale consistently degrades performance. Using Qwen3-VL-2B or Qwen2.5-VL-3B as the policy model lowers the average Acc by 5.5\% and 3.3\%, respectively, indicating that a stronger policy model is important for learning reliable tool synergy. Second, reducing the training data scale also harms performance: using only 10\% and 50\% of the training data decreases average Acc by 2.8\% and 0.9\%, and average F1 by 4.0\% and 1.4\%, respectively. This suggests that sufficient training data helps the policy model better learn when and how to combine diagnostic tools.

\begin{figure*}[h]
      \centering
      \includegraphics[width=0.85\textwidth]{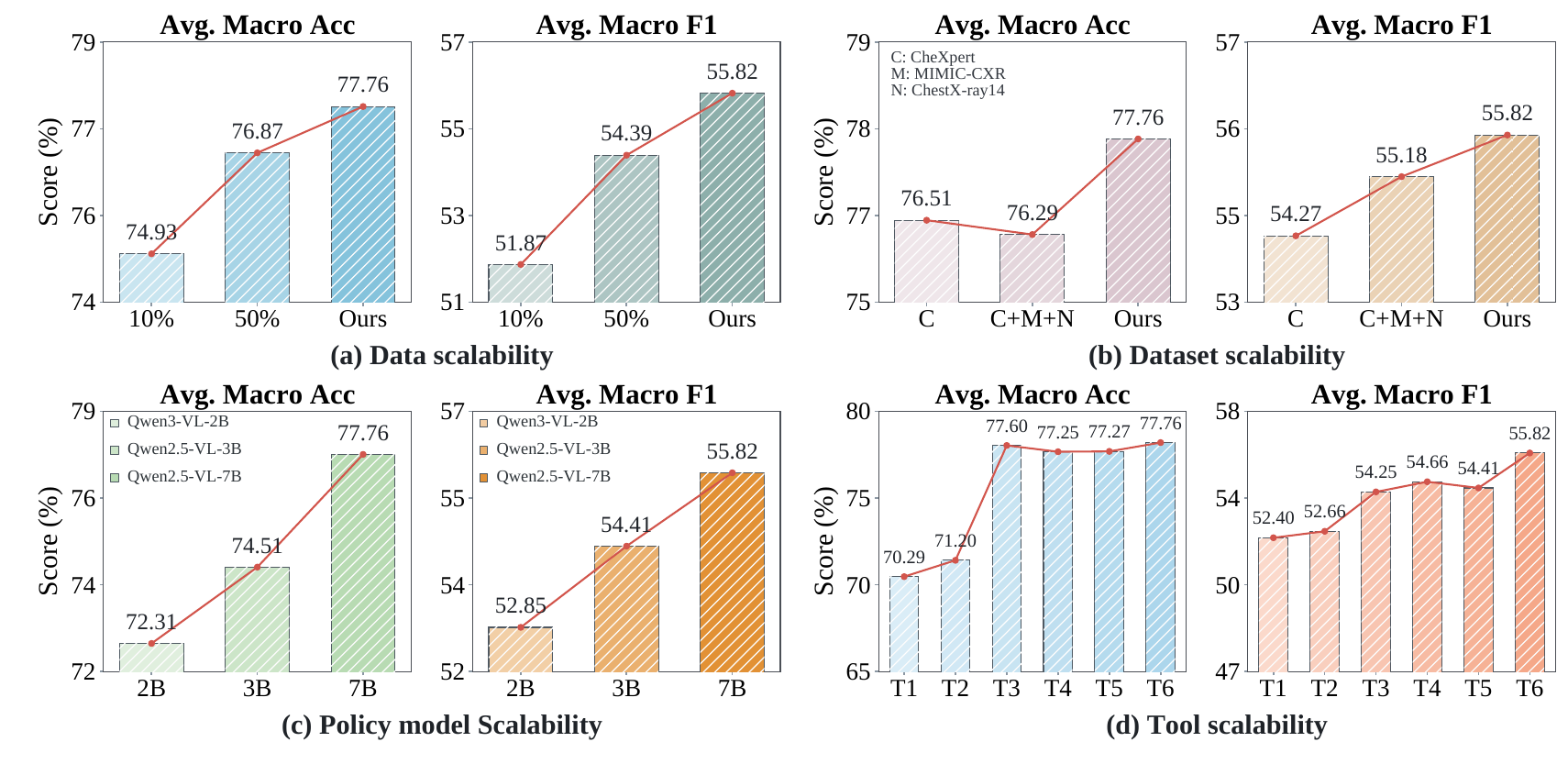}

      \caption{Scalability studies of CSRL. (a) Scalability across training data ratios; (b) Scalability across different dataset compositions; (c) Scalability across various policy models. (d) Scalability across tool-set sizes. T1-T6 denote tool sets with 1 to 6 tools, respectively.}
      \label{fig:scalability-ablation}

      \phantomsubcaption\label{fig:data_scalability}
      \phantomsubcaption\label{fig:dataset_scalability}
      \phantomsubcaption\label{fig:policy_model_scalability}
      \phantomsubcaption\label{fig:tool_scalability}
  \end{figure*}

  \vspace{-6pt}

The ablation on training dataset composition further shows that dataset diversity contributes to generalization. Training only on CheXpert or additionally including ChestX-ray14 both lead to lower average performance than the default CSRL setting, suggesting that simply adding more data is not always sufficient; the training composition must also align with the target generalization setting. The sampling and reward design ablation confirms the importance of both proposed components. Removing entropy-guided sampling decreases the average Acc and F1 by 2.2\% and 4.5\%, while removing the override reward reduces them by 0.9\% and 1.9\%, showing that exploration and disagreement-aware reward shaping are both beneficial for recovering tool synergy. Finally, the tool-call budget ablation reveals a clear performance-cost trade-off. Allowing only one or two tools substantially reduces performance, while using three to five tools approaches the full model. This indicates that multiple tools are necessary to capture complementary diagnostic signals, but most of the benefit can be obtained with a moderate tool budget. Together, these results validate the contribution of the policy model, training data, reward design, and tool budget in enabling CSRL to learn robust tool synergy.

\section{Conclusion and Limitations}
Motivated by the gap between task-level and case-level tool selection, we presented CSRL (Collaborative Synergy Reinforcement Learning), a GRPO-based framework for learning instance-level synergy among imperfect medical tools. Across chest X-ray diagnosis and medical VQA, CSRL consistently outperforms individual tools, static combinations, and general-purpose or medical multimodal models, showing that synergistic gains arise from case-level complementarity rather than standalone accuracy. \textbf{Limitations.} Our study focuses on chest X-ray benchmarks, predefined tool pools, and retrospective evaluation. Future work will extend CSRL to broader clinical tasks, dynamic tool ecosystems, and prospective human-in-the-loop validation before deployment.

\bibliographystyle{abbrv}
\bibliography{neurips_2026}

\appendix
\newpage
\section{Theoretical Proof}
\begin{proof}[Proof of Proposition \ref{prop1}]
 By definition, the Brier reward is
\[
\begin{aligned}
R_{\mathrm{brier}}(\hat p_\theta,y)
&=
1-(\hat p_\theta-y)^2 .
\end{aligned}
\]
Taking expectation over \((x,q,y)\sim\mathcal D\), we obtain
\[
\begin{aligned}
\mathbb{E}_{(x,q,y)\sim\mathcal D}
\left[
R_{\mathrm{brier}}(\hat p_\theta,y)
\right]
&=
\mathbb{E}_{(x,q,y)\sim\mathcal D}
\left[
1-(\hat p_\theta-y)^2
\right] \\
&=
1-
\mathbb{E}_{(x,q,y)\sim\mathcal D}
\left[
(\hat p_\theta-y)^2
\right] \\
&=
1-\mathcal R(\pi_\theta).
\end{aligned}
\]
Therefore,
\[
\arg\max_\theta
\mathbb{E}
\left[
R_{\mathrm{brier}}(\hat p_\theta,y)
\right]
=
\arg\min_\theta
\mathcal R(\pi_\theta).
\]
\end{proof}

\begin{proof}[proof of Proposition \ref{prop2}]
Accoring to the definition of $R_{\mathrm{override}}$, the decision made by $s$ can be presented as 
\[
\Pr(p = y) = \Pr(\ell(p,y)<s), \Pr(p \ne y) = \Pr(\ell(p,y)\geq s)
\]
Then we can define the risk of the policy as
\[
\mathcal R(\pi_\theta)
=
\Pr(\hat p_\theta \ne y),
\]
and the risk of the invoked-tool majority is
\[
\mathcal R(m)
=
\Pr(m \ne y).
\]
By definition, the override reward is
\[
R_{\mathrm{override}}(\hat p_\theta,m,y)
=
\alpha\mathbb{I}\{\hat p_\theta=y,\ m\ne y\}
-
\alpha\mathbb{I}\{\hat p_\theta\ne y,\ m=y\},
\]
where \(\alpha>0\) is the override coefficient. Taking expectation gives
\[
\begin{aligned}
\mathbb{E}
\left[
R_{\mathrm{override}}(\hat p_\theta,m,y)
\right]
&=
\alpha
\Pr(\hat p_\theta=y,\ m\ne y)
-
\alpha
\Pr(\hat p_\theta\ne y,\ m=y).
\end{aligned}
\]
On the other hand,
\[
\begin{aligned}
\mathcal R(m)-\mathcal R(\pi_\theta)
&=
\Pr(m\ne y)-\Pr(\hat p_\theta\ne y) \\
&=
\Pr(m\ne y,\hat p_\theta=y)
+
\Pr(m\ne y,\hat p_\theta\ne y) \\
&\quad -
\Pr(\hat p_\theta\ne y,m=y)
-
\Pr(\hat p_\theta\ne y,m\ne y) \\
&=
\Pr(m\ne y,\hat p_\theta=y)
-
\Pr(m=y,\hat p_\theta\ne y).
\end{aligned}
\]
Therefore,
\[
\mathbb{E}
\left[
R_{\mathrm{override}}(\hat p_\theta,m,y)
\right]
=
\alpha
\left[
\mathcal R(m)-\mathcal R(\pi_\theta)
\right].
\]
Since \(\alpha>0\), maximizing the expected override reward is equivalent to maximizing the binary risk improvement of the policy over the invoked-tool majority.

Now consider the improvement gap
\[
\mathcal G(\pi_\theta)
=
\mathcal R^\star_{\mathrm{single}}
-
\mathcal R(\pi_\theta) = \mathcal R^\star_{\mathrm{single}}
-
\mathcal R(m) + \frac{1}{\alpha}\mathbb{E}
\left[
R_{\mathrm{override}}(\hat p_\theta,m,y)
\right].
\]
Because \(\mathcal R^\star_{\mathrm{single}}\)  and \(\mathcal R(m)\)is independent of \(\theta\), maximizing \(\mathcal G(\pi_\theta)\) is equivalent to maximizing \(R_{\mathrm{override}}(\hat p_\theta,m,y)\). 
\end{proof}

\section{Tool-Calling Rollout Procedure}
\label{app:rollout}

This appendix provides the general rollout template used by CSRL. A rollout
begins with a fixed system prompt that contains the task input $x$, the query
$q$, and a \ToolsTok{} block declaring the JSON schemas of the $K$ available
tools. The policy $\pi_\theta$ then interacts with the tool-augmented
environment for up to $T_{\max}$ assistant turns.

At each assistant turn, the policy may either invoke one or more tools or
terminate the rollout with a final answer. A tool invocation follows the
template:
\[
\ToolCallTok
\{
\texttt{"name"}: T_k,\ 
\texttt{"arguments"}: \mathbf{a}_k
\}
\EndToolCallTok .
\]
The field \texttt{"name"} specifies the selected tool, and
\texttt{"arguments"} contains the tool-specific arguments constructed from the
current input, query, and dialogue context. Up to $M$ tool calls may be generated
in a single assistant turn. All tool calls generated in the same turn are
executed in parallel by the environment, and their responses are inserted into
the dialogue context before the next assistant turn.

For each invoked tool $T_k$, the environment returns its observation using the
template:
\[
\ToolRespTok{}\,
T_k\ \text{\texttt{response}}:
\ o_k
\,\EndToolRespTok{} ,
\]
where $o_k$ denotes the tool observation returned by $T_k$. The observation
format is determined by the corresponding tool schema and task setting. CSRL
does not require all tools to share the same output format; the policy conditions
on the returned observations through the dialogue context.

The policy may terminate the rollout by generating a final answer:
\[
\AnswerTok{} \hat{p}_\theta \EndAnswerTok{} ,
\]
where $\hat{p}_\theta \in \mathcal{P}$ denotes the task-specific prediction
parsed from the terminal answer. The rollout ends once a terminal answer is
produced or the maximum number of assistant turns is reached.

Tokens inserted by the environment, namely the \ToolRespTok{} blocks, are
excluded from the policy-gradient loss. Only tokens autoregressively generated
by $\pi_\theta$, including reasoning tokens, tool-call tags, and the terminal
answer, contribute to gradient updates. This ensures that optimization is
applied only to model-generated decisions rather than to environment-generated
tool observations.

\section{Broader Impact}
This work may improve the reliability of medical AI agents by learning how to combine complementary diagnostic tools on a case-by-case basis, potentially reducing missed findings and improving robustness across chest X-ray datasets and clinical settings. It can also help researchers better understand tool failure patterns and design safer medical decision-support systems. Potential risks include inheriting biases or label noise from pretrained models and datasets, and overreliance on automated outputs if used without clinical oversight. Therefore, the system should be positioned as an assistive tool for clinicians rather than a replacement for medical professionals.

\section{Training Details}
\label{app:training}

\paragraph{Hyperparameters.}
Table~\ref{tab:hparams} reports the full GRPO training configuration.
Our implementation is based on VERL~0.7.1 with vLLM~0.11.0 as the
rollout backend.

\begin{table}[h]
\centering
\small
\setlength{\tabcolsep}{6pt}
\renewcommand{\arraystretch}{1.08}
\caption{GRPO training hyperparameters.}
\label{tab:hparams}
\begin{tabular}{@{}ll@{}}
\toprule
\textbf{Hyperparameter} & \textbf{Value} \\
\midrule
\multicolumn{2}{@{}l}{\textit{Model and data}} \\
\quad Base model                   & Qwen2.5-VL-7B-Instruct \\
\quad Train batch size             & $256$ \\
\quad Max prompt / response length & $1024$ / $2048$ \\
\midrule
\multicolumn{2}{@{}l}{\textit{Optimization}} \\
\quad Learning rate                & $1\times 10^{-6}$ (constant) \\
\quad PPO mini-batch size          & $256$ \\
\quad PPO micro-batch size / GPU   & $8$ \\
\quad PPO epochs                   & $1$ \\
\midrule
\multicolumn{2}{@{}l}{\textit{GRPO objective}} \\
\quad Advantage estimator          & GRPO \\
\quad Rollouts per prompt ($n$)    & $16$ \\
\quad Clip ratio ($\epsilon$)      & $0.2$ \\
\quad KL coefficient               & $1\times 10^{-3}$ \\
\quad KL type                      & Low-variance \\
\quad Entropy coefficient          & $1\times 10^{-3}$ \\
\quad KL in reward                 & False \\
\midrule
\multicolumn{2}{@{}l}{\textit{Rollout backend}} \\
\quad Rollout backend              & vLLM (async) \\
\quad Tensor parallel size         & $2$ \\
\quad GPU memory utilization       & $0.70$ \\
\midrule
\multicolumn{2}{@{}l}{\textit{Multi-turn tool use}} \\
\quad Max assistant turns          & $4$ \\
\quad Max parallel tool calls      & $6$ \\
\quad Number of tools              & $6$ \\
\midrule
\multicolumn{2}{@{}l}{\textit{Schedule and hardware}} \\
\quad Total epochs                 & $2$ \\
\quad Save / validation frequency  & Every $5$ steps \\
\quad Hardware                     & $8\times$ NVIDIA H200 141\,GB \\
\bottomrule
\end{tabular}
\end{table}

\paragraph{System Prompt.}
The following system prompt is prepended to every training and evaluation
sample.

\begin{quote}
\itshape
You are a radiology AI expert for chest X-ray diagnosis. You have access
to 6 diagnostic tools with different strengths. Each tool returns a score for the queried condition. Query MULTIPLE tools
and compare their results for a more reliable diagnosis. Tools may
disagree---use the disagreement to assess diagnostic uncertainty. After
consulting the tools, provide your final diagnosis probability ($0.0$ to
$1.0$) inside \texttt{<answer></answer>} tags.
\end{quote}

\section{LLM Usage Statement}
\label{app:llm_usage}

During the preparation of this manuscript, large language models were used only
to polish the writing and provide suggestions on figure presentation. The
authors remain responsible for all technical content, experimental results, and
scientific conclusions.


\end{document}